\documentclass[letterpaper, 10 pt, conference]{ieeeconf}  

\usepackage{subcaption}
\usepackage{caption}
\usepackage{mathptmx} 
\usepackage{times} 
\usepackage{amsmath} 

\usepackage{amssymb}  
\usepackage{mathrsfs}                                  
\usepackage{graphics}
\usepackage{graphicx} 
\usepackage{float}
\usepackage{soul}
\usepackage{comment}
\usepackage{algorithmicx}
\usepackage{algpseudocode}
\usepackage[ruled,vlined]{algorithm2e}
\usepackage[english]{babel}
\usepackage{tikz}
\usepackage{ftnxtra}
\usepackage{color}
\DeclareMathAlphabet{\mathcal}{OMS}{cmsy}{m}{n}
\newtheorem{remark}{Remark}
\newtheorem{theorem}{Theorem}

\newtheorem{definition}{Definition}

\newif\ifdraft
\draftfalse

\IEEEoverridecommandlockouts                              

\title{\LARGE \bf 
	An Invariant-EKF VINS Algorithm for Improving Consistency	
}

\author{Kanzhi Wu$ ^{\star}$, Teng Zhang$ ^{\star}$, Daobilige Su, Shoudong Huang and Gamini Dissanayake
\thanks{$ ^{\star} $ indicates the equal contributions of the first two authors.}
\thanks{Kanzhi Wu, Teng Zhang, Daobilige Su, Shoudong Huang and Gamini Dissanayake  are with the Center for Autonomous Systems, University of Technology Sydney, Ultimo, NSW 2007, Australia.
	 {\tt \small  \{Kanzhi.Wu, Teng.Zhang, Daobilige.Su, Shoudong.Huang, Gamini.Dissanayake\}@uts.edu.au}
  }  
}

\begin{document}

\maketitle

\begin{abstract}
The main contribution of this paper is an invariant extended Kalman filter (EKF) for visual inertial navigation systems (VINS). It is demonstrated that the conventional EKF based VINS is not invariant under \textit{the stochastic unobservable transformation}, associated with translations and a rotation about the gravitational direction. This can lead to inconsistent state estimates as the estimator does not obey a fundamental property of the physical system. To address this issue, we use a novel uncertainty representation to derive a Right Invariant error extended Kalman filter (RIEKF-VINS) that preserves this invariance property. RIEKF-VINS is then adapted to the multi-state constraint Kalman filter framework to obtain a consistent state estimator.
Both Monte Carlo simulations and real-world experiments are used to validate  the proposed method.
\end{abstract}

\section{INTRODUCTION}

Visual-Inertial Navigation Systems (VINS)  have been of significant interest to   the robotics community in the past decade, as
the  fusion of information from a camera and an inertial measurement unit (IMU) provides an effective and affordable  solution for navigation in GPS-denied environment. VINS algorithms can be classified into two categories, namely,  filter based  and  optimization based. Although 
there has been recent progress in the development of optimization based algorithms \cite{forster2015imu}\cite{okvis}, the extended Kalman filter (EKF) based solutions are still   extensively used (e.g., \cite{msckf}\cite{inertialSLAM2007}\cite{rovio}\cite{IMUcausal}) mainly as a result of their  efficiency and  simplicity. 

It is well known that conventional  EKF based Simultaneous Localization and Mapping algorithms (EKF-SLAM) \cite{Julier-2001}\cite{HuangT-RO2007}  suffer from  inconsistency. Similarly it has been shown that the conventional EKF VINS algorithm (ConEKF-VINS) using  point features in the environment is also inconsistent  resulting in the underestimation of the state uncertainty. This is closely related to the partial observability of these systems because conventional EKF algorithms
do not necessarily guarantee this fundamental property \cite{HuangT-RO2007}\cite{ocvins1} due to the linearized errors, which is the main reason for the overconfident estimates.   
This insight has been a catalyst for a number of  \textit{observability-constraint}  algorithms  (e.g., \cite{FEJ-VINS}\cite{hesch2013camera}\cite{OC-VINS2}), that explicitly enforces the unobservability of the system along specific directions via  the modifications to the Jacobian matrices. 
Although the \textit{observability-constraint}  algorithms improve 
the consistency and accuracy of the estimator to some extent \cite{askOCVINS}, extra computations in the update stage are required.
 Bloesch et al. in \cite{rovio}   propose a robot-centric formulation to alleviate the inconsistency. 
 Under the robot-centric formulation, the filter estimates the locations of landmarks in the local frame instead of that in the global frame. As a result,  the system becomes fully observable so that this issue is inherently avoided.
However, this formulation can result in  larger uncertainty and extra computations in the propagation stage, as discussed in \cite{GPHuang-consistencyIJRR}\cite{TengRIEKF}.  

Recently,  the manifold and Lie group representations for three-dimensional orientation/pose have been utilized for solving SLAM and VINS. Both  filter based algorithms (e.g., \cite{TengRIEKF}\cite{mahony2008nonlinear}\cite{carlone2015quaternion}) and  optimization based algorithms (e.g., \cite{forster2015imu}\cite{CoP}) can  benefit from the manifold representation and better accuracy can be achieved. The use of manifold  does not only allow much easier algebraic  computations (e.g., the computation of the Jacobian matrices) and avoid the representation singularity \cite{murray1994mathematical} but also have inspired a number of researchers to rethink the difference between the state representation and the state uncertainty representation, which is highlighted in  \cite{forster2015imu}\cite{TengRIEKF}. In fact, this insight is also 
intrinsically understood in  the well-known preintegration visual-inertial algorithm  \cite{Lupton} although the algorithm does not use  the manifold representation.
From  the viewpoint of control theory,  Aghannan and Rouchon in \cite{Bonnabel2002} propose 
a  framework for designing  symmetry-preserving observers on manifolds  by using a subtle geometrically adapted correction term. 
The fusion of the symmetry-preserving theory and EKF has resulted in  the invariant-EKF (I-EKF), which possesses the theoretical
local convergence property \cite{iekfstable} and preserves the same invariance property of the original system.  I-EKF based observers have been used in the inertial navigation \cite{iekfInertial} and the 2D EKF-SLAM \cite{Bonnabel2012}\cite{BarrauB15}. Our recent work \cite{TengRIEKF} also proves the significant improvement in the consistency through  
 a 3D I-EKF SLAM algorithm.  

In this paper,
we argue that 
the absence of the invariance affects the consistency of ConEKF-VINS estimates. There is a correspondence between this and the observability analysis reported in the previous literatures (e.g.,  \cite{IMUcausal}\cite{hesch2013camera}). The invariance in this 
refers to ``the output of the filter is invariant under any \textit{stochastic unobservable transformation}". For the VINS system, the unobservable transformation is the rotation about the gravitational direction and the translations. 
Adopting the I-EKF framework, we propose the Right Invariant error EKF VINS algorithm (RIEKF-VINS) and prove that it is invariant.
We then integrate RIEKF-VINS into the well-known visual-inertial odometry framework, i.e.,  the multi-state constraint Kalman filter (MSCKF) and remedy the inconsistency of the MSCKF algorithm. 
We  show using extensive Monte Carlo simulations the proposed method outperforms the original MSCKF, especially in terms of the consistency.
A preliminary real-world experiment also demonstrates the improved accuracy of the proposed method.

This paper is organized as follows. Section \ref{Section::stdEKFVINS} recalls the VINS system and gives an introduction of the ConEKF-VINS under the general continuous-discrete EKF. 
 Section \ref{Section::Consis} performs the consistency analysis of the general EKF algorithm based  on the invariance theory and proves the absence of the invariance of ConEKF-VINS. 
Section \ref{Section::RIEKF} proposes RIEKF-VINS with the extension to the MSCKF framework. 
Section \ref{Section::SimulationAndExp}  reports both the simulation and experiment results.
 Finally, Section \ref{Section::Conclusion} includes the main conclusions of this work and future work. Appendix provides some necessary formulas used in the proposed algorithms and the proofs of the theorems.

\textbf{Notations:}  Throughout this paper bold lower-case and upper-case letters are reserved for \textit{column vectors}  and matrices/tuples, respectively. To simplify the presentation, the vector transpose operators are omitted for the case $\mathbf{A}=\begin{bmatrix}
\mathbf{a}^\intercal, \mathbf{b}^\intercal ,\cdots , \mathbf{c}^\intercal 
\end{bmatrix}^\intercal$. 
The notation ${S}(\cdot )$ denotes the skew symmetric operator that transforms a 3-dimensional vector into a skew symmetric matrix: ${S}(\mathbf{x})\mathbf{y}=\mathbf{x} \times \mathbf{y}$ for $\mathbf{x}, \mathbf{y} \in\mathbb{R}^3$, where the notation $\times$ refers to the cross product.

\section{Background Knowledge}
\label{Section::stdEKFVINS}
In this section, we first provide an overview of the VINS system and then describe the ConEKF-VINS algorithm  based on the framework of the general continuous-discrete EKF.

\subsection{The VINS system}
The VINS system is used to estimate the state denoted as the tuple below
\begin{equation}
\mathbf{X}=\left ( \mathbf{R}, \mathbf{v}, \mathbf{p}, \mathbf{b}_g, \mathbf{b}_a,  \mathbf{f} \right)
\label{eq::X}
\end{equation}
where $ \mathbf{R} \in \mathbb{SO}(3) $ and $ \mathbf{p} \in \mathbb{R}^{3}$ are the orientation and  position of the IMU sensor, respectively,
$\mathbf{v} \in \mathbb{R}^{3}$ is the IMU velocity expressed in the global frame, $\mathbf{b}_g\in \mathbb{R}^{3}$ is the gyroscope bias, $\mathbf{b}_a\in \mathbb{R}^{3}$ is the accelerometer bias and 
 $ \mathbf{f} \in \mathbb{R}^{3}$ is the coordinates of the landmark  in the global frame. Note that only one landmark is included    
 the system state (\ref{eq::X})   for a more concise notation.  

\subsubsection{The continuous-time motion model}
The IMU measurements are usually used for state evolution due to its high frequency. 
The continuous-time motion model of the VINS system is given by the following  ordinary differential equations (ODEs): 
\begin{equation}
\begin{aligned}
&\dot{\mathbf{X}} = f( \mathbf{X}, \mathbf{u}, \mathbf{n}) \\
 =&  \left( \mathbf{R}S(\mathbf{w} - \mathbf{b}_g  -  \mathbf{n}_g ), \mathbf{R}(\mathbf{a}-\mathbf{b}_a - \mathbf{n}_a )+\mathbf{g} , \mathbf{v}, \mathbf{n}_{bg} , \mathbf{n}_{ba}  , \mathbf{0} \right)
\end{aligned}
\label{eq::VINSmotionmodel}
\end{equation}
where  $\mathbf{w}\in\mathbb{R}^3$ is the gyroscope reading,
 $ \mathbf{a} \in \mathbb{R}^3$ is the  accelerometer reading, 
  $\mathbf{g} \in \mathbf{R}^3$ is the global gravity vector (constant),
 and $\mathbf{n}= \begin{bmatrix}
\mathbf{n}_g, \mathbf{n}_{bg} , \mathbf{n}_a , \mathbf{n}_{ba}
\end{bmatrix}$ is  the system noise modeled as a white Gaussian noise with the covariance matrix $\mathbf{Q}$:  $ \mathbb{E}(\mathbf{n} (t)  \mathbf{n}(\tau)^\intercal) =  \mathbf{Q} \delta(t-\tau)$.
 Note that  $\mathbf{u}= ( \mathbf{w}, \mathbf{a}, \mathbf{g} )$ is the time-varying system \textit{input} and the IMU noise covariance $\mathbf{Q}$ is a constant matrix as prior knowledge.

\subsubsection{The discrete-time measurement model}
The visual measurement as the  system \textit{output} is discrete due to the low frequency of camera. 
After data association and rectification, 
 the visual measurement of the landmark at time-step $k\in \mathbb{N}$  is available and given by
\begin{equation}
	\mathbf{z}_k = h (\mathbf{X}_k, \mathbf{n}_z)    =  \mathfrak{h} (  \mathbf{R}_k^\intercal (\mathbf{f} -\mathbf{p}_k ) ) + \mathbf{n}_z
	\label{eq::VINSmeasurementmode}
\end{equation}   
 where $\mathbf{n}_z \sim \mathcal{N}(\mathbf{0}, \mathbf{V}_k)$ is the measurement noise. Note that
 $ \mathfrak{h} (\cdot) := \pi \circ \mathbf{T}_{CI}  $, where $\pi$ denotes the projection function and
     $\mathbf{T}_{CI}$ is the  transformation  from  the IMU frame to the camera frame.

\subsection{The general continuous-discrete EKF}

 Being a natural extension of the standard EKF, the general EKF allows more flexible uncertainty representation by the following:
\begin{equation}
\mathbf{X}=\hat{\mathbf{X}}\oplus \mathbf{e} \text{ and     } \mathbf{e} \sim  \mathcal{N}(\mathbf{0}, \mathbf{P} )
\end{equation}
 where $(\hat{\mathbf{X}}, \mathbf{P})$ can be  regarded as   the mean estimate and the covariance matrix,   $\mathbf{e}$ is 
a white Gaussian noise vector and the  notation $\oplus$ is  called retraction in differentiable geometry \cite{absil2007trust}, coupled with the inverse mapping $\ominus$. 
Note that  the \textit{user-defined} operators $\oplus$ and $\ominus$ need to be designed such that $\mathbf{X}=\mathbf{X}\oplus \mathbf{0}$ and  $\mathbf{e}=\mathbf{X}\ominus\hat{\mathbf{X}}$. Here we also highlight that the choice of the retraction $\oplus$ has a significant contribution to the performance of the filter, as discussed in our previous work \cite{TengRIEKF}. 

Once determining the retraction $\oplus$,  the process of the general continuous-discrete EKF
is similar to conventional continuous-discrete EKF, as summarized in  Alg. \ref{Alg::cd-EKF}.  
For propagation, we first calculate the \textit{time-varying} Jacobians matrices $\mathbf{F}$ and $\mathbf{G}$  from the linearized error-state propagation model:
 \begin{equation}
\dot{\mathbf{e}} = \mathbf{F} \mathbf{e} + \mathbf{G} \mathbf{n} + o(\| \mathbf{e} \| \| \mathbf{n} \|).
 \end{equation}
We then compute the \textit{state transition matrix} $\pmb{\Phi}_n:=\pmb{\Phi}(t_{n+1}, t_n) $ that  is the solution at time $t_{n+1}$ of the following ODE:
   \begin{equation}
\frac{d}{dt} { \pmb{\Phi} } (t , t_n ) = \mathbf{F}(t) \pmb{\Phi}(t,t_n)
   \end{equation}
 with the  condition $\pmb{\Phi}(t_n, t_n) = \mathbf{I}$ at time $t_n$. The matrix $\mathbf{Q}_{d,n}$ can be computed as
 \begin{equation}
\mathbf{Q}_{d,n}= \int_{t_n}^{t_{n+1}}  \pmb{\Phi}(t_{n+1}, \tau ) \mathbf{G}(\tau) \mathbf{Q} \mathbf{G}^\intercal(\tau)  \pmb{\Phi}^\intercal(t_{n+1}, \tau )  d\tau.
 \end{equation}  

\begin{algorithm}[]
	\KwIn{$\hat{\mathbf{X}}_n$, $\mathbf{P}_n$, $\mathbf{u}_{t_n : t_{n+1}}$, $\mathbf{z}_{n+1}$;     }
	\KwOut{   $\hat{\mathbf{X}}_{n+1}$, $\mathbf{P}_{n+1}$;   }
	\textbf{Propagation:} \\
	Turn off the system noise and 
	compute  $\hat{\mathbf{X}}_{n+1|n}$ with  $\hat{\mathbf{X}}_n$ and the ODEs (\ref{eq::VINSmotionmodel}); \\
$
\mathbf{P}_{n+1|n} \leftarrow   \pmb{\Phi}_n \mathbf{P}_{n}  \pmb{\Phi}^\intercal_n + \mathbf{Q}_{d,n}
$; \\
	\textbf{Update:} \\
	 $\mathbf{H}_{n+1}=\frac{\partial  h( \hat{\mathbf{X} }_{n+1|n} \oplus \mathbf{e}, \mathbf{0}  )  }{\partial \mathbf{e}}|_{\mathbf{e}=\mathbf{0} } $; \\
	$ \mathbf{S}\leftarrow \mathbf{H}_{n+1} \mathbf{P}_{n+1|n}   \mathbf{H}_{n+1}^\intercal+\mathbf{V}_{n+1}$,   
	$\mathbf{K} \leftarrow \mathbf{P}_{n+1|n} \mathbf{H}_{n+1}^\intercal \mathbf{S}^{-1} $;\\ 
	$\tilde{\mathbf{z}} \leftarrow h(\hat{\mathbf{X}}_{n+1|n}, \mathbf{0} )-\mathbf{z}_{n+1} $; \\
	$\hat{\mathbf{X}}_{n+1} \leftarrow  \hat{\mathbf{X}}_{n+1|n} \oplus  \mathbf{ K}\tilde{\mathbf{z}}    $, ${\mathbf{P}}_{n+1} \leftarrow (\mathbf{I}-\mathbf{KH}_{n+1})\mathbf{P}_{n+1|n}$;\\
	\caption{The general continuous-discrete EKF}
	\label{Alg::cd-EKF}
\end{algorithm}

\subsection{ConEKF-VINS}

ConEKF-VINS  \cite{ocvins1} can be regarded as an instance of the general EKF algorithm (Alg. \ref{Alg::cd-EKF}). 
In ConEKF-VINS,  the uncertainty  representation is  defined as
\begin{equation}
\begin{aligned}
	&\mathbf{X}=\hat{\mathbf{X}} \oplus \mathbf{e}   \\
	=& \left ( \hat{\mathbf{R}} \exp( \mathbf{e}_\theta ) , \hat{\mathbf{v}}+\mathbf{e}_v, \hat{\mathbf{p}}+\mathbf{e}_p, \hat{\mathbf{b}}_g+\mathbf{e}_{bg}, \hat{\mathbf{b}}_a+\mathbf{e}_{ba},  \hat{\mathbf{f}}+ \mathbf{e}_f \right)
\end{aligned}
\label{eq::stdVINSoplus}
\end{equation}
where 
$ \mathbf{e} =  \begin{bmatrix}
\mathbf{e}_\theta  , \mathbf{e}_v, \mathbf{e}_p ,\mathbf{e}_{bg}, \mathbf{e}_{ba}   , \mathbf{e}_f
\end{bmatrix} \sim \mathcal{N}(\mathbf{0},\mathbf{P}) $ and $\exp(\cdot)$ transforms a 3-dimensional vector into a rotation matrix, given in (\ref{eq::expS}).  The matrices  $\pmb{\Phi}_n$, $\mathbf{Q}_{d,n} $ and $\mathbf{H}_{n+1}$ 
are omitted here due to space reasons, which can be straightforwardly calculated in the sense of the uncertainty representation (\ref{eq::stdVINSoplus}).
Please  refer to  \cite{ocvins1} for more details.

\section{Consistency Analysis}
\label{Section::Consis}

In this section, we first introduce the concepts of unobservable transformation,  invariance and observability. 
We then perform the consistency analysis for the general EKF filter and  prove that ConEKF-VINS  does \textit{not} have the expected invariance property. Moreover, we also discuss the relationship between invariance and consistency.

\subsection{Unobservability, unobservable transformation and invariance of the VINS system}

The concept observability of nonlinear systems can be traced to the early literature \cite{hermann1977nonlinear}.
As discussed in the literatures \cite{IMUcausal}\cite{CIBijrr}\cite{martinelli2013observability},
 the state   (\ref{eq::X}) of the VINS system is not locally observable. To make it more intuitive, we introduce the unobservability of the VINS system  based on the \textit{unobservable transformation} rather than   the \textit{observability rank criterion}  reported in \cite{hermann1977nonlinear}. 

\begin{definition}
The transformation $\mathcal{T}$ is  called to be an \textbf{unobservable transformation} for the VINS system and 
the output of the VINS system (\ref{eq::X})--(\ref{eq::VINSmeasurementmode}) is \textbf{invariant} under $\mathcal{T}$ when the following condition is satisfied:
For arbitrary $t_i$ such that $ \mathbf{Y}(t_i)=  \mathcal{T} (\mathbf{X}(t_i)) $, we have 
$h(\mathbf{X}(t_n), \mathbf{0} ) = h(\mathbf{Y}(t_n), \mathbf{0})$  $\forall\text{ } n \geq i $ where the notations 
$\mathbf{X}(\cdot)$ and $\mathbf{Y}(\cdot)$ denote the two evoluted  trajectories that follow the same ODEs  (\ref{eq::VINSmotionmodel}) with the conditions $\mathbf{X}(t_i)$ and $\mathbf{Y}(t_i)$ at time $t_i$, respectively. 
On the other hand, the system is called to be unobservable if there exists an \textit{unobservable transformation}.
\end{definition}

\begin{remark}
One can see that an unobservable system is always accompanied by an \textit{unobservable transformation}. And the invariance to   the  \textit{unobservable transformation} is a more detailed description of the unobservability. 
\end{remark}

\begin{definition}
	\label{def:un}
	For the system state (\ref{eq::X}), a \textbf{stochastic transformation of translation and rotation (about the gravitational direction)} $\mathcal{T}_{\mathbf{S}}$  is a mapping:
	\begin{equation}
	\begin{aligned}
	&\mathcal{T}_{\mathbf{S}} (\mathbf{X})=( \exp(\mathbf{g} ( \epsilon_1+\theta_1 ))\mathbf{R}, 
	\exp(\mathbf{g}( \epsilon_1+\theta_1 ))\mathbf{v},\\& \exp(\mathbf{g}( \epsilon_1+\theta_1 )) \mathbf{p}+ \pmb{\theta}_2+\pmb{\epsilon}_2,\\
	&  \mathbf{b}_g, \mathbf{b}_a,   \exp(\mathbf{g}( \epsilon_1+\theta_1 )) \mathbf{f}+ \pmb{\theta}_2+\pmb{\epsilon}_2     )
	\end{aligned}
	\label{eq::st}
	\end{equation}	
	where  
	 $\mathbf{S}= (  \pmb{\theta}, \pmb{\epsilon} )$, $\theta_1 \in \mathbb{R}$, $\pmb{\theta}_2 \in \mathbb{R}^3$, $\pmb{\theta}= \begin{bmatrix}
    \theta_1 , \pmb{\theta}_2
	\end{bmatrix} \in \mathbb{R}^4$, $\epsilon_1 \in \mathbb{R}$, $\pmb{\epsilon}_2\in \mathbb{R}^3$ and 
	$\pmb{\epsilon}= \begin{bmatrix}
	\epsilon_1 , \pmb{\epsilon}_2
	\end{bmatrix} \in \mathbb{R}^4 $ is a white Gaussian noise with the covariance $\pmb{\Sigma}$.
	    $\mathcal{T}_{\mathbf{S}}$ degenerates into the \textbf{deterministic  transformation $\mathcal{T}_{\mathbf{D}}$} ($\mathbf{D}= (  \pmb{\theta}, \mathbf{0} )$) under the condition ${\pmb{\Sigma}}=\mathbf{0}$. $\mathcal{T}_{\mathbf{S}}$ degenerates into a \textbf{stochastic identity transformation} under the condition 
	    $ \pmb{\theta} = \mathbf{0}$,  .
\end{definition}

\begin{theorem}
	\label{theorem::vins}
The stochastic transformation 	$\mathcal{T}_{\mathbf{S}}$ is an unobservable transformation to the VINS system (\ref{eq::X})--(\ref{eq::VINSmeasurementmode}). 
\end{theorem}
\begin{proof}
It can be straightforwardly verified.
\end{proof}
 
\begin{remark}
 Theorem \ref{theorem::vins} corresponds to the conclusion in \cite{IMUcausal}\cite{martinelli2012vision} that the IMU yaw angle and the IMU position  are (locally) unobservable. 
\end{remark}

\subsection{The invariance of the general EKF based filter}

The general EKF based filter is \textit{not} a linear system for the estimated state $\hat{\mathbf{X}}$. However, the invariance of the filter  can be described as the following:
\begin{definition}
	The output  of a general EKF framework based filter (Alg. \ref{Alg::cd-EKF}) for the VINS system is  invariant under any \textbf{stochastic unobservable transformation} $\mathcal{T}_{\mathbf{S}}$ if the  following condition is satisfied:
	  for any  two  estimates $(\hat{\mathbf{X}}_i, \mathbf{P}_i)$ and  $ ( \hat{\mathbf{Y}}_i, \mathbf{P}y_i ) $ at time-step $i$, where  $ \hat{\mathbf{Y}}_i= \mathcal{T}_{\mathbf{S}} (  \hat{\mathbf{X}}_i ) $ and $ \mathbf{P}y_i= \mathbf{M}_i \mathbf{P}_i\mathbf{M}^\intercal_i+ \mathbf{N}_i {\pmb{\Sigma}}\mathbf{N}^\intercal_i  $ in which  \begin{equation}
	\begin{aligned}
	\mathbf{M}_i&:=\left.\frac{\partial  \mathcal{T}_{{\mathbf{D}}} (\hat{\mathbf{X}}_i \oplus \mathbf{e}  ) \ominus  \mathcal{T}_{\mathbf{D}} (\hat{\mathbf{X}}_i )  }{\partial \mathbf{e}}\right\rvert_{\mathbf{e}=\mathbf{0} }
	\end{aligned}
	\end{equation}
	and
	\begin{equation}
		\mathbf{N}_i:=\left.\frac{\partial  \mathcal{T}_{\mathbf{S}} (\hat{\mathbf{X}}_i ) \ominus  \mathcal{T}_{\mathbf{D}} (\hat{\mathbf{X}}_i )  }{\partial \pmb{\epsilon}  }\right\rvert_{\pmb{\epsilon}=\mathbf{0} } 
	\end{equation}
	we have $h(\hat{\mathbf{X}}_{n},\mathbf{0})=h(\hat{\mathbf{Y}}_{n},\mathbf{0})$ for all $n \geq i$. The notations  
	$\hat{\mathbf{X}}_{n}$ and $\hat{\mathbf{Y}}_{n}$ above represent the \textit{mean} estimate of this  filter at time-step $n$
	by using the same input $\mathbf{u}$ from time $t_i$ to $t_n$,  from  the  conditions $(\hat{\mathbf{X}}_i, \mathbf{P}_i)$ and $(\hat{\mathbf{Y}}_i, \mathbf{P}y_i)$ at time-step $i$,  respectively. 
\end{definition}

As shown in Def. 1 and Def. 2,  the invariance to any stochastic transformation $\mathcal{T}_{\mathbf{S}}$ can be divided into two properties: 1) \textbf{the invariance to any deterministic  transformation $\mathcal{T}_{\mathbf{D}}$} and 2) \textbf{the invariance to any stochastic identity transformation}.
The following two theorems  analytically provide the methods to judge whether a general EKF based filter has the  two invariances properties above.

\begin{theorem}
	 \label{theorem:invariancedet}
 The output of the general EKF based filter for the VINS system is invariant under any  deterministic unobservable  transformation only if for each  deterministic unobservable  transformation $\mathcal{T}_{\mathbf{D}}$, there exists an invertible matrix $\mathbf{W}_{\mathbf{D}}$ (unrelated to $\mathbf{X}$)  such that
 \begin{equation}
\mathcal{T}_{\mathbf{D}}(\mathbf{X} \oplus \mathbf{e} ) =\mathcal{T}_{\mathbf{D}}( \mathbf{X} ) \oplus  \mathbf{W}_{\mathbf{D}}\mathbf{e}.
\label{eq::deter}
 \end{equation}
\end{theorem}
\begin{proof}
	See Appendix \ref{proof::determinsticVINS}.
\end{proof}

\begin{theorem}
		\label{theorem:invariancesto} 
	 The output of the general EKF based filter for the VINS system is invariant under any  stochastic identity  transformation only if  
	 \begin{equation}
	 \mathbf{H}_{n+i+1}\pmb{\Phi}_{n+i}\pmb{\Phi}_{n+i-1}\cdots\pmb{\Phi}_{i} \mathbf{N}_i = \mathbf{0} \text{ } \forall\text{ } n\text{ and } i\geq 0.
	 \label{eq::sto}
	 \end{equation}
\end{theorem}
\begin{proof}
	See Appendix \ref{proof::stochasticVINS}.
\end{proof}
By using the theorems above, we can easily determine the invariance  properties of ConEKF-VINS. 
\begin{theorem}
 ConEKF-VINS  satisfies (\ref{eq::deter})  but does not satisfies (\ref{eq::sto}). Hence,  ConEKF-VINS has the invariance to any deterministic unobservable  transformation $\mathcal{T}_{\mathbf{D}}$  but not the invariance to  stochastic identity  transformations.  In all, the output of  ConEKF-VINS is not invariant under {stochastic unobservable transformation} $\mathcal{T}_{\mathbf{S}}$.
\end{theorem}
\begin{proof}
For the ConEKF-VINS algorithm,	the invariance  to the deterministic unobservable transformation $\mathcal{T}_{\mathbf{D}}$ can be verified by using Theorem \ref{theorem:invariancedet}. The absence of invariance of  ConEKF-VINS to stochastic identity  transformations can be verified by using Theorem \ref{theorem:invariancesto}. More details are omitted here.
\end{proof}

\begin{remark}
The previous literatures \cite{ocvins1}\cite{FEJ-VINS}\cite{hesch2013camera}\cite{OC-VINS2} directly perform the observability analysis of the filter  on the \textit{linearized error-state model}. 
 However, Theorem \ref{theorem:invariancedet} and Theorem \ref{theorem:invariancesto}  clarifies the relationship between the  filter and  the \textit{linearized error-state model}.  
\end{remark}

\subsection{Consistency and invariance}
\label{Section::criterion}

The unobservability  in terms of stochastic unobservable transformation $\mathcal{T}_{\mathbf{S}} $ is a fundamental property of the VINS system. Therefore a consistent filter (as a system for the estimated state $\hat{\mathbf{X}}$) is expected to mimic this property, i.e.,  {\bf the  output of a consistent estimator is invariant under any stochastic unobservable transformation.}  The invariance to  to the deterministic transformation $\mathcal{T}_{\mathbf{D}} $  implies  that the estimates from the filter do \textit{not} depend on the selection of the (initial)  mean estimate of the unobservable variables, i.e., the IMU yaw angle and the IMU position, essentially.
Similarly, the invariance to stochastic identity transformation implies that the uncertainty w.r.t these unobservable variables does not affect the subsequent  \textit{mean} estimates.  
We can conclude that \textbf{the consistency of a filter is tightly coupled with the invariance to stochastic unobservable transformation}.  
A  filter that does not have the invariance property will gain the unexpected information and  produce inconsistent (overconfident) estimates.  Note that  ConEKF-VINS is a typical example due to the absence of the invariance property.

\section{The Proposed Method: 
	 RIEKF-VINS}
\label{Section::RIEKF}

In this section, we propose RIEKF-VINS by using a new uncertainty representation and prove it has the expected  invariance properties. We then apply  RIEKF-VINS to the MSCKF framework.

\subsection{The Uncertainty representation and Jacobians}
RIEKF-VINS also follows the  framework  (Alg. \ref{Alg::cd-EKF}). 
The uncertainty representation of  RIEKF-VINS is defined as below  
\begin{equation}
\begin{aligned}
\mathbf{X}=&\hat{\mathbf{X}} \oplus \mathbf{e}   \\
=&( \exp( \mathbf{e}_\theta ) \hat{\mathbf{R}} , \exp(\mathbf{e}_{\theta})\hat{\mathbf{v}}+J_r(-\mathbf{e}_\theta)\mathbf{e}_v, 
 \\
&\exp(\mathbf{e}_{\theta}) \hat{\mathbf{p}}+J_r(-\mathbf{e}_\theta)\mathbf{e}_p,  \hat{\mathbf{b}}_g+\mathbf{e}_{bg}, \hat{\mathbf{b}}_a+\mathbf{e}_{ba}, \\
&\exp(\mathbf{e}_{\theta})   \hat{\mathbf{f}}+J_r(-\mathbf{e}_\theta) \mathbf{e}_f )
\end{aligned}
\label{eq::RIEKFVINSoplus}
\end{equation}
where  
$ \mathbf{e} =  \begin{bmatrix}
\mathbf{e}_\theta, \mathbf{e}_v , \mathbf{e}_p , \mathbf{e}_{bg}, \mathbf{e}_{ba} , \mathbf{e}_f
\end{bmatrix} \sim \mathcal{N}(\mathbf{0},\mathbf{P})$  and the right Jacobian operator  $J_r(\cdot)$ is  given in (\ref{eq::Jr}). Note that this uncertainty representation intrinsically employs the Lie group so that  the recent result (Theorem 2 of \cite{iekfstable}) can be used to easily compute the Jacobians $\mathbf{F}$ and $\mathbf{G}$ of the propagation
\begin{equation}
\mathbf{F} =  \begin{bmatrix}
\mathbf{0}_{3,3} & \mathbf{0}_{3,3} & \mathbf{0}_{3,3} & -\hat{\mathbf{R}} &\mathbf{0}_{3,3} & \mathbf{0}_{3,3}\\
S(\mathbf{g}) & \mathbf{0}_{3,3} &\mathbf{0}_{3,3} & -S(\hat{\mathbf{v}})\hat{\mathbf{R}}  & -\hat{\mathbf{R}} & \mathbf{0}_{3,3}  \\
\mathbf{0}_{3,3} & \mathbf{I}_{3} &\mathbf{0}_{3,3} & -S(\hat{\mathbf{p}})\hat{\mathbf{R}} & \mathbf{0}_{3,3} &\mathbf{0}_{3,3}  \\
\mathbf{0}_{3,3} & \mathbf{0}_{3,3} & \mathbf{0}_{3,3} & \mathbf{0}_{3,3} &\mathbf{0}_{3,3} &\mathbf{0}_{3,3} \\
\mathbf{0}_{3,3} & \mathbf{0}_{3,3} & \mathbf{0}_{3,3} & \mathbf{0}_{3,3} &\mathbf{0}_{3,3} &\mathbf{0}_{3,3} \\
\mathbf{0}_{3,3} & \mathbf{0}_{3,3} & \mathbf{0}_{3,3} & \mathbf{0}_{3,3} &\mathbf{0}_{3,3} &\mathbf{0}_{3,3} \\
\end{bmatrix}
\label{eq::RIEKF_F}
\end{equation} 
and  
\begin{equation}
	\mathbf{G} =  \begin{bmatrix}
	\hat{\mathbf{R}} & \mathbf{0}_{3,3}  & \mathbf{0}_{3,3} & \mathbf{0}_{3,3}  \\
	S(\hat{\mathbf{v}})\hat{\mathbf{R}} &\mathbf{0}_{3,3}  & \hat{\mathbf{R}} & \mathbf{0}_{3,3} \\
	S(\hat{\mathbf{p}})\hat{\mathbf{R}} & \mathbf{0}_{3,3}  & \mathbf{0}_{3,3} & \mathbf{0}_{3,3} \\
		\mathbf{0}_{3,3} & \mathbf{I}_{3} & \mathbf{0}_{3,3} & \mathbf{0}_{3,3}  \\
		\mathbf{0}_{3,3} & \mathbf{0}_{3,3} & \mathbf{0}_{3,3} & \mathbf{I}_{3} \\
		S(\hat{\mathbf{f}})\hat{\mathbf{R}} & \mathbf{0}_{3,3}  & \mathbf{0}_{3,3} & \mathbf{0}_{3,3} \\
	\end{bmatrix}.
	\label{eq::RIEKF_G}
\end{equation}
The measurement Jacobian  is
\begin{equation}
	\mathbf{H}_{n+1} = \partial \mathfrak{h} ( \hat{\mathbf{f}}_{n+1,I}) \begin{bmatrix}
	\mathbf{0}_{3,6} &   -\hat{\mathbf{R}}_{n+1|n}^\intercal & \mathbf{0}_{3,6} & \hat{\mathbf{R}}_{n+1|n}^\intercal
	\end{bmatrix}
\end{equation}
where $ \hat{\mathbf{f}}_{n+1,I} =\hat{\mathbf{R}}_{n+1|n}^\intercal(\hat{\mathbf{f}}_{n+1|n} - \hat{\mathbf{p}}_{n+1|n} )    \in \mathbb{R}^3 $.

\subsection{Invariance proof}
 
\begin{theorem}
The output of RIEKF-VINS  is invariant under any stochastic unobservable transformation $\mathcal{T}_{\mathbf{S}}$.
\end{theorem}
\begin{proof}
For the retraction defined in (\ref{eq::RIEKFVINSoplus}), we have $\mathcal{T}_{\mathbf{D}}(\mathbf{X} \oplus \mathbf{e} ) =\mathcal{T}_{\mathbf{D}}( \mathbf{X} ) \oplus  \mathbf{W}_{\mathbf{D}}\mathbf{e}$ $\forall$ $\mathbf{X}$ and $\mathbf{e}$, where 
\begin{equation}
\mathbf{W}_{\mathbf{D}} = \begin{bmatrix}
\delta \mathbf{R} &  \mathbf{0}_{3,3} & \mathbf{0}_{3,3} & \mathbf{0}_{3,3} &\mathbf{0}_{3,3} & \mathbf{0}_{3,3}   \\
\mathbf{0}_{3,3}                   &  \delta \mathbf{R}  & \mathbf{0}_{3,3} & \mathbf{0}_{3,3} &\mathbf{0}_{3,3} & \mathbf{0}_{3,3}      \\
S(\pmb{\theta}_2)\delta \mathbf{R} & \mathbf{0}_{3,3} &\delta \mathbf{R}      & \mathbf{0}_{3,3} & \mathbf{0}_{3,3} &\mathbf{0}_{3,3}                        \\
\mathbf{0}_{3,3} & \mathbf{0}_{3,3} &\mathbf{0}_{3,3}    & \mathbf{I}_{3} & \mathbf{0}_{3,3} &\mathbf{0}_{3,3}                            \\	 
\mathbf{0}_{3,3} & \mathbf{0}_{3,3} &\mathbf{0}_{3,3}    & \mathbf{0}_{3,3} & \mathbf{I}_{3} &\mathbf{0}_{3,3}                            \\	
S(\pmb{\theta}_2)\delta \mathbf{R} & \mathbf{0}_{3,3} &\mathbf{0}_{3,3}    & \mathbf{0}_{3,3} & \mathbf{0}_{3,3} &\delta \mathbf{R}                           \\		                              
\end{bmatrix}
\end{equation}
and $\delta \mathbf{R} : = \exp(\mathbf{g}\theta_1)$. According to Theorem \ref{theorem:invariancedet},  the output of RIEKF-VINS is invariant  under any deterministic transformation  $\mathcal{T}_{\mathbf{D}}$.
On the other hand, for all  $i$, we have 
\begin{equation}
\pmb{\Phi}_{i} = 
\begin{bmatrix}
\mathbf{I}_3 &  * &  \mathbf{0}_{3,3} & *  & * &  \mathbf{0}_{3,3}    \\
\Delta t_{i} S(\mathbf{g}) & *  &  \mathbf{0}_{3,3} & *  & * &  \mathbf{0}_{3,3}  \\
\frac{\Delta t_{i}^2}{2} S(\mathbf{g}) & *  & \mathbf{I}_3  & *  & * & \mathbf{0}_{3,3} \\ 
\mathbf{0}_{3,3}   & * & \mathbf{0}_{3,3}  &  *  & * &  \mathbf{0}_{3,3}  \\
\mathbf{0}_{3,3}   & * & \mathbf{0}_{3,3}  &  * & * &  \mathbf{0}_{3,3}\\
\mathbf{0}_{3,3}   & * & \mathbf{0}_{3,3}  &  * & * & \mathbf{I}_3   \\
\end{bmatrix}
\label{eq::Phi}
\end{equation}
and 
\begin{equation}
\mathbf{N}_i =  \left.\frac{\partial  \mathcal{T}_{\mathbf{S}} (\hat{\mathbf{X}}_i ) \ominus  \mathcal{T}_{\mathbf{D}} (\hat{\mathbf{X}}_i )  }{\partial \pmb{\epsilon}  }\right\rvert_{\pmb{\epsilon}=\mathbf{0} } = \begin{bmatrix}
\mathbf{g} & \mathbf{0}_{3,3} \\
\mathbf{0}_{3,1}  &  \mathbf{0}_{3,3}\\
\mathbf{0}_{3,1}  &  \mathbf{I}_{3}\\
\mathbf{0}_{3,1}  &  \mathbf{0}_{3,3}\\
\mathbf{0}_{3,1}  &  \mathbf{0}_{3,3}\\
\mathbf{0}_{3,1}  &  \mathbf{I}_3
\end{bmatrix}
\label{eq::N}
\end{equation}
where $\Delta t_{i}:= t_{i+1}-t_{i}$ and the elements denoted by the notation $*$ are omitted here because these do not have any contribution to the computation of $\pmb{\Phi}_{i} \mathbf{N}_{i}$.
Note that $\pmb{\Phi}_{i} \mathbf{N}_{i}  = \mathbf{N}_{i+1} $
and $\mathbf{H}_{i+1}\mathbf{N}_{i+1} = \mathbf{0} $
 for all $ i$ and then we can easily verify that RIEKF-VINS  satisfies (\ref{eq::sto}). According to Theorem \ref{theorem:invariancesto}, the output of RIEKF-VINS is invariant under any stochastic identity transformation.
\end{proof}
\begin{remark}
The \textit{observability-constraint} filters proposed in  \cite{ocvins1}\cite{FEJ-VINS}\cite{hesch2013camera}\cite{OC-VINS2}  artificially modify the transition matrix $\pmb{\Phi}_n$ and the measurement Jacobian $\mathbf{H}_{n+1}$ to meet the condition (\ref{eq::sto}) such that they have the invariance to stochastic identity transformation. As a comparison, our proposed RIEKF-VINS employs the uncertainty representation (\ref{eq::RIEKFVINSoplus}) such that 
 the  ``natural"  matrices $\pmb{\Phi}_n$ and  $\mathbf{H}_{n+1}$  can elegantly  meet the condition (\ref{eq::sto}). 
\end{remark}

\subsection{Application to MSCKF}
A drawback of ConEKF-VINS and RIEKF-VINS is the expensive cost of maintaining  the covariance matrix for a number of landmarks. Especially,
RIEKF-VINS suffers from the  complexity quadratic to the number of landmarks in the propagation stage.
On the other hand, the well known MSCKF \cite{msckf} that has the complexity linear to the number of landmarks  inherits the inconsistency of ConEKF-VINS. 
One can see that  the uncertainty w.r.t the global yaw has effects on the mean estimates in the MSCKF algorithm, unexpectedly.
Due to the reasons above,  we integrate RIEKF-VINS into the MSCKF framework such that the modified algorithm has the linear complexity and better consistency.   
For convenience, we call the modified filter as RI-MSCKF.
In this subsection, we do not state all details of RI-MSCKF but  point out the modifications.

\subsubsection{System state and retraction}
The system state $\mathcal{X}_n$ at time-step $n$ in RI-MSCKF is 
\begin{equation}
\begin{aligned}
\mathcal{X}_n= ( \bar{\mathbf{X}}_n, \mathbf{C}_{t_1}, \cdots, \mathbf{C}_{t_j}, \cdots,   \mathbf{C}_{t_k},\cdots,\mathbf{C}_{t_m} )
\end{aligned}
\end{equation}
where
$\bar{\mathbf{X}}_n =  ( \mathbf{R}_n, \mathbf{v}_n, \mathbf{p}_n, \mathbf{b}_{g,n}, \mathbf{b}_{a,n}   ) $ denotes the IMU state at time-step $n$,
 $\mathbf{C}_{t_i} = ( \mathbf{R}^c_{t_i}, \mathbf{p}^c_{t_i}  )\in \mathbb{SE}(3)$ denotes the camera pose at the time $t_i$ ($t_i < t_n $).
According to the IMU state uncertainty in RIEKF-VINS, 
the uncertainty representation of $\mathcal{X}_n$ are defined as below
\begin{equation}
\begin{aligned}
\mathcal{X}_n &= \hat{\mathcal{X}}_n \oplus \mathbf{e}  \\
& = ( \hat{\bar{\mathbf{X}}}_n  \oplus_{imu} \mathbf{e}_{I}, \hat{\mathbf{C}}_{t_1} \oplus_{pose} \mathbf{e}_c^1,\cdots,  \hat{\mathbf{C}}_{t_m} \oplus_{pose} \mathbf{e}_c^m  )\\
\end{aligned}
\label{eq::RIMSCKFoplus}
\end{equation}
where 
 $ \mathbf{e} =  \begin{bmatrix}
\mathbf{e}_{I} , \mathbf{e}_{c}
\end{bmatrix} \in \mathbb{R}^{15+6m} \sim \mathcal{N}(\mathbf{0},\mathbf{P}_n)$,  $\mathbf{e}_{I} \in \mathbb{R}^{15} $ and $\mathbf{e}_c = \begin{bmatrix}
\mathbf{e}_c^1, \cdots, \mathbf{e}_c^m
\end{bmatrix} \in \mathbb{R}^{6m}$. Note that  $\oplus_{imu}$ and $\oplus_{pose}$ are given in Appendix \ref{subsection::formulation}.  

\subsubsection{Propagation}
The mean propagation $\mathcal{X}_{n+1|n}$ of RI-MSCKF also follows that  of MSCKF while  the covariance $\mathbf{P}_{n+1|n}$ is calculated by 
\begin{equation}
\mathbf{P}_{n+1|n} =\bar{\pmb{\Phi}}_n^\intercal \mathbf{P}_{n} \bar{\pmb{\Phi}}_n+ \bar{\mathbf{Q}}_{d,n}
\end{equation}
where $\bar{\pmb{\Phi}}_n = \textbf{Diag}(\pmb{\Phi}^{I}_n,\mathbf{I}_{6m} )$, $\bar{\mathbf{Q}}_{d,n} = \textbf{Diag}(\mathbf{Q}^I_{d,n},\mathbf{0}_{6m,6m} )$. 
Note that  $\pmb{\Phi}^{I}_n$ and $\mathbf{Q}^I_{d,n}$ are the matrices from the first 15 rows and 15 columns of $\pmb{\Phi}_n$ and
$\mathbf{Q}_{d,n}$, respectively, 
where $\pmb{\Phi}_n$ and $\mathbf{Q}_{d,n}$ are the matrices of RIEKF-VINS.

\subsubsection{State augment}
Once  a new image is captured at time-step $n+1$, we augment the system state and the covariance matrix as the following: 
\begin{equation}
\begin{aligned}
\hat{\mathcal{X}}_{n+1|n} \leftarrow  ( \hat{\mathcal{X}}_{n+1|n}, \hat{\mathbf{C}}_{t_{n+1}}  )
\end{aligned}
\end{equation}
\begin{equation}
\begin{aligned}
\mathbf{P}_{n+1|n} \leftarrow   \begin{bmatrix}
\mathbf{I}_{15+6m} \\
\mathbf{J}
\end{bmatrix} \mathbf{P}_{n+1|n} \begin{bmatrix}
\mathbf{I}_{15+6m} \\
\mathbf{J}
\end{bmatrix}^\intercal
\end{aligned} 
\end{equation}
where $\hat{\mathbf{C}}_{t_{n+1}} =  (  \hat{\mathbf{R}}_{{n+1|n}} \Delta \mathbf{R} , \hat{\mathbf{R}}_{n+1|n} \Delta \mathbf{p}+ \hat{\mathbf{p}}_{n+1|n}       )   \in\mathbb{SE}(3)$ is
 the \textit{mean} estimate of camera pose at the time $t_{n+1}$, $(\Delta \mathbf{R}, \Delta \mathbf{p})\in\mathbb{SE}(3)$  denotes the transformation from the camera to the IMU. Due to the new uncertainty representation (\ref{eq::RIMSCKFoplus}), the Jacobian  $\mathbf{J}$ needs to be changed as below
 \begin{equation}
 	\mathbf{J}= \begin{bmatrix}
 	\mathbf{I}_3 & \mathbf{0}_{3,3} & \mathbf{0}_{3,3} & \mathbf{0}_{3,6} & \mathbf{0}_{3,6m} \\
 	\mathbf{0}_{3,3} & \mathbf{0}_{3,3} & \mathbf{I}_3 & \mathbf{0}_{3,6} & \mathbf{0}_{3,6m}
 	\end{bmatrix}.
 \end{equation}

\subsubsection{Update}
Note that the landmark uncertainty  is coupled with the IMU pose in RIEKF-VINS. In RI-MSCKF, 
we  describe the landmark uncertainty coupled with the camera pose $\mathbf{C}_{t_j}$ that earliest captures the landmark  within the current system state $\mathcal{X}_n$ as below
\begin{equation}
(	\hat{\mathbf{C}}_{t_j}, \hat{\mathbf{f}}) \oplus \bar{\mathbf{e}}_c^j = (\hat{\mathbf{C}}_{t_j}\oplus_{pose} \mathbf{e}_c^j , \mathbf{e}^j_{\theta} \hat{\mathbf{f}}+ J_r(- \mathbf{e}^j_{\theta})\mathbf{e}_f  )
\label{eq::rimsckf-landmark}
\end{equation}
where  $\bar{\mathbf{e}}_c^j=[ \mathbf{e}_c^j ,  \mathbf{e}_f ]=  [ \mathbf{e}^j_{\theta},\mathbf{e}^j_{p}, \mathbf{e}_f  ]\in \mathbf{R}^9$.
From the uncertainty representations (\ref{eq::RIMSCKFoplus}) and (\ref{eq::rimsckf-landmark}), we can compute the linearized measurement model for the visual measurement at time-step $k$ ($t_1 \leq  t_k \leq t_n$). With a slight abuse of notations, the
 linearized measurement model can be represented as below
\begin{equation}
\begin{aligned}
	\pi( \hat{\mathbf{R}}^{c \intercal}_{t_k} (\hat{\mathbf{f}} - \hat{\mathbf{p}}^c_{t_k} ) )-\mathbf{z}_{k} & \approx \partial \pi_k \mathbf{H}^*_{xk} \mathbf{e}_{n+1|n} +\partial \pi_k \mathbf{H}^*_{fk} \mathbf{e}_{f}+\mathbf{V}_{k} \\
	\tilde{\mathbf{z}}_{k} & \approx \partial \pi_k \mathbf{H}^*_{xk} \mathbf{e}_{n+1|n} +\partial \pi_k \mathbf{H}^*_{fk} \mathbf{e}_{f} +\mathbf{V}_{k} \\
		\tilde{\mathbf{z}}_{k} & \approx   \mathbf{H}_{xk} \mathbf{e}_{n+1|n} + \mathbf{H}_{fk} \mathbf{e}_{f} +\mathbf{V}_{k}
		\label{eq::zHxHf}
\end{aligned}
\end{equation}
where   $\partial \pi_k:= \partial \pi ( \hat{\mathbf{R}}_{t_k}^{c \intercal } (\hat{f} -\mathbf{p}^c_{t_k} )  )$, $\mathbf{z}_k$ is the measurement captured at the time $t_k$. Here the matrices  $\mathbf{H}^*_{fk}$ and $ \mathbf{H}^*_{xk} $  are given by
\begin{equation}
	\mathbf{H}^*_{fk}  =  \hat{\mathbf{R}}^{c \intercal}_{t_{k}} \text{ }\text{ and}
\end{equation}
 \begin{equation}
		\begin{aligned}
		\mathbf{H}^*_{xk} & = \begin{bmatrix}
		\cdots  &	\cdots & \cdots & \mathbf{A}  &\cdots& \mathbf{B} & \cdots & \cdots
		\end{bmatrix}  \\ 
		\end{aligned}
\end{equation}
where $\mathbf{A}= \begin{bmatrix}
- \hat{\mathbf{R}}^{c \intercal}_{t_{k}} S(\hat{\mathbf{f}}),\mathbf{0}_{3,3}
\end{bmatrix}$ and $\mathbf{B}= \begin{bmatrix}
\hat{\mathbf{R}}^{c \intercal}_{t_{k}}S(\hat{\mathbf{f}}) , - \hat{\mathbf{R}}^{c \intercal}_{t_{k}}
\end{bmatrix}$.
Due to the absence of the covariance of landmark, RI-MSCKF also uses the null-space trick on (\ref{eq::zHxHf}) and  the resulting
 residual equation 
\begin{equation}
	\begin{aligned}
 \mathbf{H}_{fk}^\perp \tilde{\mathbf{z}}_{k} & \approx   \mathbf{H}_{fk}^\perp \mathbf{H}_{xk} \mathbf{e}_{n+1|n}   +\mathbf{H}_{fk}^\perp \mathbf{V}_{k}  \\
 \tilde{\mathbf{z}}'_{k} & \approx    \mathbf{H}'_{xk} \mathbf{e}_{n+1|n} +  \mathbf{V}'_{k}
	\end{aligned}
\end{equation}
 is employed for update.
\begin{remark}
RI-MSCKF does not need any extra computation to maintain the expected invariance while the \textit{observability-constraint} algorithms need to explicitly project the  measurement Jacobians onto the \textit{observable space}. 
\end{remark}

\section{Simulation and Experiment}
\label{Section::SimulationAndExp}

\subsection{Simulation Result}

In order to validate the theoretical contributions in this paper, we perform 50 Monte Carlo simulations and compare RI-MSCKF to MSCKF for a Visual-Inertial Odometry (VIO) scenario without loop closure.
 
\begin{figure}[htbp]
	\centering
    \includegraphics[width=1\linewidth]{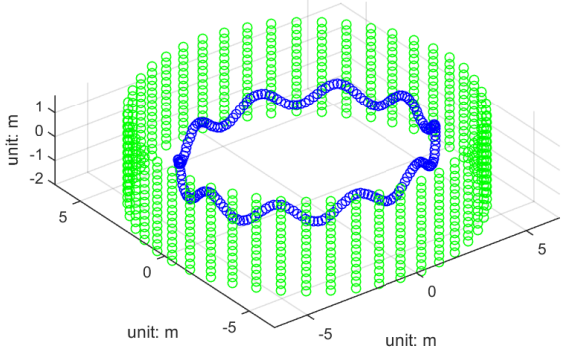}
	\caption{The simulated trajectory (blue circles) and landmarks (green stars).}
	\label{fig::VINS_simulation_circle}
\end{figure}

\begin{figure*}[htbp]
	\centering
	\includegraphics[width=7.0in]{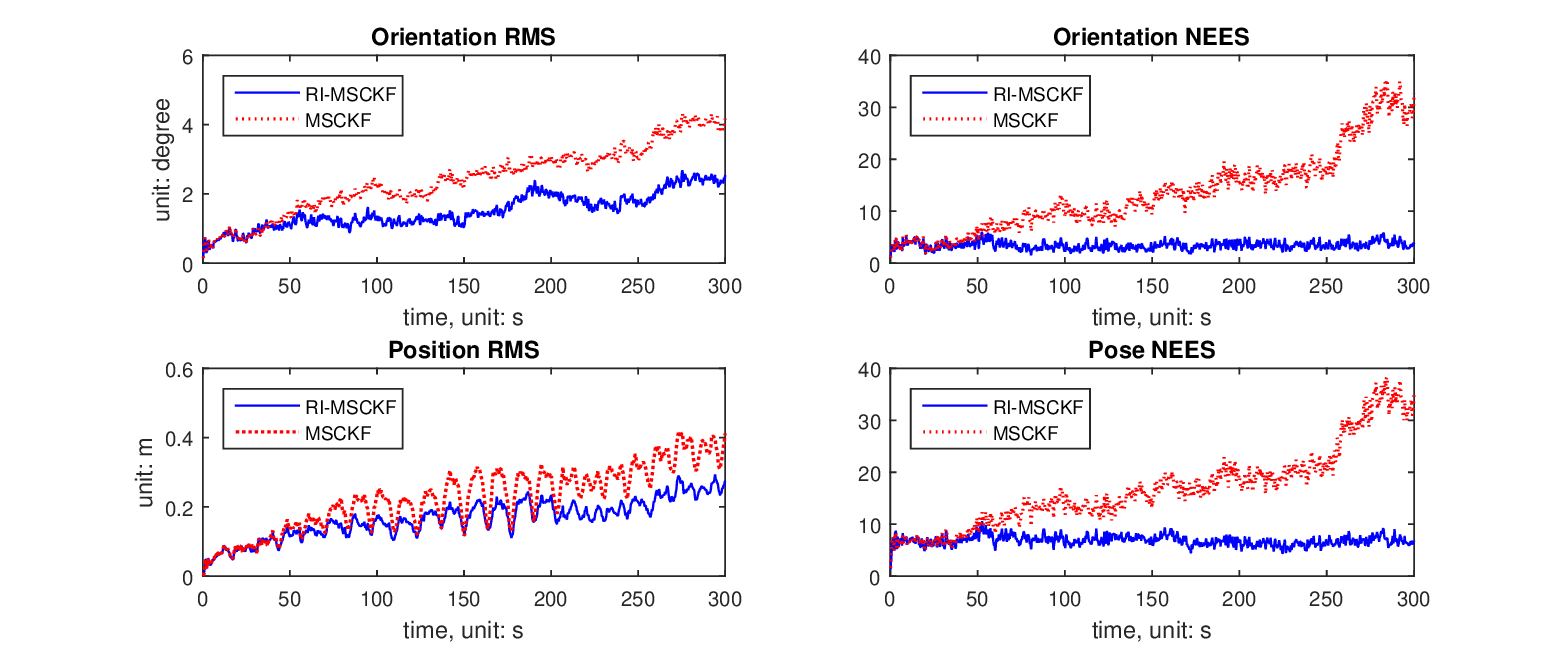}
	\caption{50 Monte Carlo simulation results. The proposed RI-MSCKF outperforms the original MSCKF, both in terms of accuracy (RMS) and consistency (NEES).}
	\label{fig::RIMSCFvsMSCKF}
\end{figure*}

Consider that a robot equipped with an IMU and a camera moves in a specific trajectory (average speed is $3$m/s) with the sufficient 6-DOFs motion, shown as the blue circles in Fig. \ref{fig::VINS_simulation_circle}. In this environment, 675 landmarks are distributed on the surface of a cylinder with radius $6.5m$ and height $4m$ shown as the green stars in Fig. \ref{fig::VINS_simulation_circle}.
Under the simulated environment, the camera is able to observe sufficiently overlapped landmarks between consecutive frames.
The standard deviation of camera measurement is set as $1.5$ pixels.
The IMU noise covariance $\mathbf{Q}$ is set as $\textbf{Diag}( 0.008^2\mathbf{I}_3,0.0004^2\mathbf{I}_3, 0.019^2\mathbf{I}_3, 0.05^2\mathbf{I}_3 )$ (the International System of Units).
 In each round of Monte Carlo simulation, the initial estimate is set as  the ground truth. And the measurements from IMU and camera are generated from the same trajectory with random noises.
 The maximal number of camera poses in the system state  of RI-MSCKF and MSCKF is set as 10. For robust estimation,  we use the landmarks for the update step only when  the landmarks are captured more than 5 times by the cameras within the current system state. 

The results of 50 Monte Carlo simulations are plotted in Fig. \ref{fig::RIMSCFvsMSCKF}. We use 
the root mean square  error (RMS) and the average normalized estimation error squared (NEES) to evaluate  both accuracy and consistency, respectively. Note that the \textit{ideal} NEES of orientation is 3 and that of pose is 6. 
As shown in Fig. \ref{fig::RIMSCFvsMSCKF}, RI-MSCKF clearly outperforms MSCKF especially for the consistency.
This phenomenon  can be explained as  RI-MSCKF has the invariance property to stochastic rotation about the gravitational direction and thus it can reduce the unexpected information gain when compared to MSCKF. In addition, the RMS of orientation and position of both filters increase with the time because the loop closure in this simulation is turned off.  

\subsection{Preliminary Experiment}

In order to validate the performance of the proposed RI-MSCKF algorithm under practical environments, we evaluate the algorithm on Euroc dataset \cite{euroc} which is collected on-board a macro aerial vehicle in the indoor environments. Without a delicated designed front-end which handles the feature extraction and tracking perfectly, we selected sequence \textit{V2\_01\_easy} in this section to demonstrate the performance of the RI-MSCKF algorithm where the features can be tracked correctly and thus making it perfect to compare our algorithm against the  MSCKF algorithm.

In this preliminary experiment, we designed a front-end based on ORB-SLAM \cite{mur2015orb} while only keeping the feature tracking sub-module. Without knowing the map points, new keyframe is inserted once there is $ n_{\text{frames}} $ frames have passed  since the insertion of the last keyframe. One sample image with the tracked landmarks is shown in Fig. \ref{fig::VINSsampled}.
 The uncertainty of the IMU sensor is set as instructed in the dataset. The maximal number of the camera poses in the system state is set as $ 10 $ and the minimal observed times for a landmark is set as $ 5 $. 

\begin{figure}[htbp]
	\centering
	\includegraphics[width=1\linewidth]{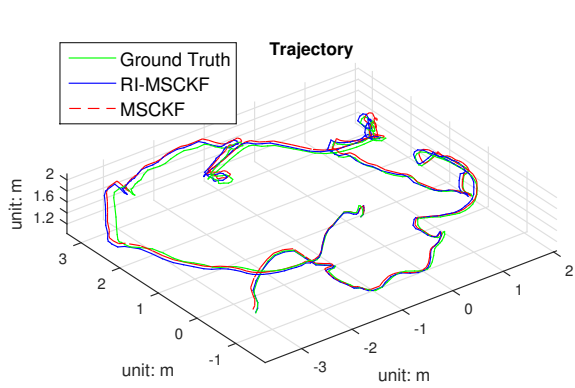}
	\caption{The estimated trajectories from MSCKF and RI-MSCKF in \textit{V2\_01\_easy}.}
	\label{fig::VINS_euroc}
\end{figure}

\begin{figure}[htbp]
	\centering
	\includegraphics[width=0.8\linewidth]{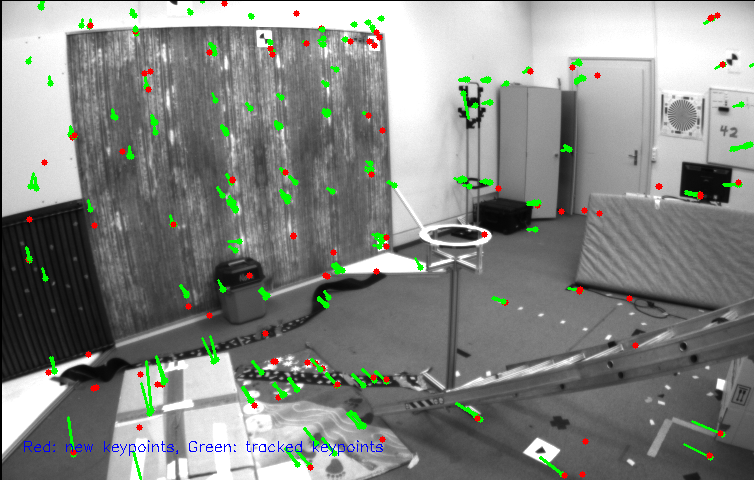}
	\caption{Sample image with landmarks in the experiment. The green dots represent the tracked key points and the red dots represents the new key points.}
	\label{fig::VINSsampled}
\end{figure}

Fig. \ref{fig::VINS_euroc} shows the estimated trajectories using MSCKF and RI-MSCKF. As shown in Fig. \ref{fig::VINS_euroc} and indicated in Fig.\ref{fig::euroc_position_rms}, RI-MSCKF shows the similar accuracy of position compared with MSCKF but also avoids the drift in the last few frames of the sequence, however,  RI-MSCKF shows significant better results in terms of orientation estimation accuracy compared with the original MSCKF algorithm. 
Even without a robust front-end to handle feature tracking perfectly, this preliminary experiment is able to demonstrate the superiority of RI-MSCKF compared with MSCKF algorithm in terms of the estimation accuracy.

\begin{figure}[htbp]
	\centering
	\includegraphics[width=1\linewidth]{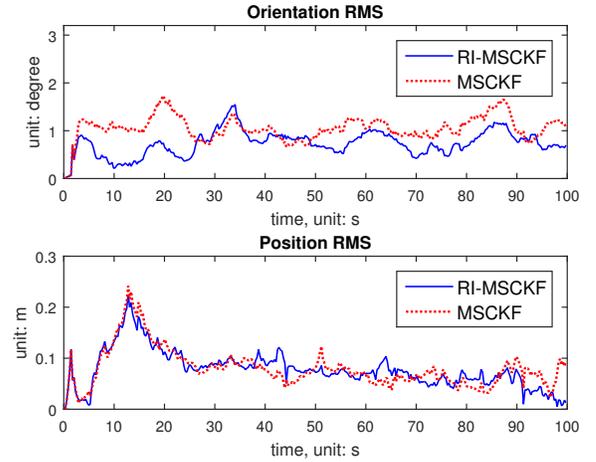}
	\caption{The RMS of orientation and  position estimate from MSCKF and RI-MSCKF in \textit{V2\_01\_easy}.}
	\label{fig::euroc_position_rms}
\end{figure}


\section{Conclusion And Future Work}
\label{Section::Conclusion}
In this work,  we proposed the  RIEKF-VINS algorithm and stressed that the consistency of a filter is tightly coupled with the invariance property.
We  proved that RIEKF-VINS has the expected invariance property while ConEKF-VINS does not satisfy  this property.
We also provided the methods to check whether a general EKF based filter has the invariance properties.
After theoretical analysis, we integrated RIEKF-VINS into the MSCKF framework such that the resulting RI-MSCKF algorithm 
can achieve better consistency relative to the original MSCKF.
Monte Carlo simulations  illustrated    the significantly improved performance of RI-MSCKF, especially for the consistency.
The real-world experiments also validated its improved accuracy. Future work includes improving the front end 
to achieve more robust estimation. We will also compare RIEKF-VINS to the \textit{observability-constraint} algorithms in both simulations and real-world experiments.

\appendix

\subsection{Some Formulas}
\label{subsection::formulation}
  \begin{equation}
\begin{aligned}
&\exp(\mathbf{y})  =\mathbf{I}_3+\frac{\sin (\| \mathbf{y} \| )}{\| \mathbf{y} \|} {S}(\mathbf{y})+ \frac{1-\cos(\| \mathbf{y} \|) }{\| \mathbf{y} \|^2}{S}^2(\mathbf{y})   \\
\end{aligned}
\label{eq::expS}
\end{equation}
\begin{equation}
	{J}_r(\mathbf{y}) =\mathbf{I}_3 -\frac{1-\cos (\| \mathbf{y} \| )}{\| \mathbf{y} \|^2} {S}(\mathbf{y})+ \frac{\| \mathbf{y} \|-\sin(\| \mathbf{y} \|) }{\| \mathbf{y} \|^3}{S}^2(\mathbf{y})
	\label{eq::Jr}
\end{equation}
for $\mathbf{y}\in\mathbb{R}^3$. 

The notation $\oplus_{imu}$ is defined as
\begin{equation}
\begin{aligned}
\bar{\mathbf{X}}  \oplus_{imu} \mathbf{e}_{I} = (\exp( \mathbf{e}_\theta ) {\mathbf{R}} , \exp(\mathbf{e}_{\theta}){\mathbf{v}}+J_r(-\mathbf{e}_\theta)\mathbf{e}_v, \\
\exp(\mathbf{e}_{\theta}) {\mathbf{p}}+J_r(-\mathbf{e}_\theta)\mathbf{e}_p,  {\mathbf{b}_g}+\mathbf{e}_{bg},{\mathbf{b}_a}+\mathbf{e}_{ba} )
\end{aligned}
\end{equation}
where $\bar{\mathbf{X}}=(\mathbf{R},\mathbf{v},\mathbf{p},\mathbf{b}_g,\mathbf{b}_a)$ and $\mathbf{e}_I= \begin{bmatrix}
\mathbf{e}_\theta,\mathbf{e}_v, \mathbf{e}_p, \mathbf{e}_{bg}, \mathbf{e}_{ba}
\end{bmatrix} \in \mathbb{R}^{15}$

The notation $\oplus_{pose}$ is defined as
\begin{equation}
\mathbf{C}\oplus_{pose} \mathbf{e}^i_c = (\exp(\mathbf{e}^i_\theta) \mathbf{R}, \exp(\mathbf{e}^i_{\theta}) {\mathbf{p}}+J_r(-\mathbf{e}^i_\theta)\mathbf{e}^i_p  )
\end{equation}  
where $\mathbf{C}=(\mathbf{R},\mathbf{p})\in\mathbb{SE}(3)$ and $\mathbf{e}^i_c=\begin{bmatrix}
\mathbf{e}^i_\theta, \mathbf{e}^i_p
\end{bmatrix}\in\mathbb{R}^6$. 
 
\subsection{Proof of Theorem \ref{theorem:invariancedet}}
\label{proof::determinsticVINS}

Here we only prove the sufficient condition. It is assumed that this filter satisfies:  for each deterministic  unobservable transformation $\mathcal{T}_{\mathbf{D}}$ there exists $\mathbf{W}_{\mathbf{D}}$ such that $\mathcal{T}_{\mathbf{D}}(\mathbf{X} \oplus \mathbf{e} ) =\mathcal{T}_{\mathbf{D}}( \mathbf{X} ) \oplus  \mathbf{W}_{\mathbf{D}}\mathbf{e}$. 

For any  estimate $(\hat{\mathbf{X}}_i,\mathbf{P}_i)$ at time-step $i$, we have another estimate $(\hat{\mathbf{Y}}_{i},\mathbf{P}y_{i})=( \mathcal{T}_{\mathbf{D}} (  \hat{\mathbf{X}}_i),  \mathbf{W}_{\mathbf{D}} \mathbf{P}_i  \mathbf{W}_{\mathbf{D}}^\intercal    )$  after applying the deterministic transformation $\mathcal{T}_{\mathbf{D}}$. 
After one step propagation, we have $(\hat{\mathbf{X}}_{i+1|i},\mathbf{P}_{i+1|i})$ and $(\hat{\mathbf{Y}}_{i+1|i},\mathbf{P}y_{i+1|i})$ where $ \hat{\mathbf{Y}}_{i+1|i}= \mathcal{T}_{\mathbf{D}} (  \hat{\mathbf{X}}_{i+1|i} ) $  and $ \mathbf{Py}_{i+1|i}=  \mathbf{W}_{\mathbf{D}} \mathbf{P}_{i+1|i}  \mathbf{W}_{\mathbf{D}}^\intercal  $. Note that $\mathbf{H}y_{i+1} = \mathbf{H}_{i+1} \mathbf{W}_{\mathbf{D}}^{-1} $ and 
then it is easy to obtain $\mathbf{K}y=\mathbf{W}_{\mathbf{D}}   \mathbf{K}  $, resulting in the mean estimate $\hat{\mathbf{Y}}_{i+1}$  as below
\begin{equation}
\begin{aligned}
\hat{\mathbf{Y}}_{i+1}&= \hat{\mathbf{Y}}_{i+1|i}\oplus \mathbf{K}_y \tilde{\mathbf{z}}\\
& =\mathcal{T}_{\mathbf{D}}( \hat{\mathbf{X}}_{i+1|i}  )  \oplus  \mathbf{W}_{\mathbf{D}} \mathbf{K}\tilde{\mathbf{z}} \\
&= \mathcal{T}_{\mathbf{D}} ( \hat{\mathbf{X}}_{i+1|i} \oplus  \mathbf{K} \tilde{\mathbf{z}}     ) \\
& =  \mathcal{T}_{\mathbf{D}}( \hat{\mathbf{X}}_{i+1}  ) 
\end{aligned}
\end{equation}
 The covariance matrix after update  becomes 
$ \mathbf{P}y_{i+1}= (\mathbf{I}- \mathbf{K}_y\mathbf{H}y_{i+1} )\mathbf{P}y_{i+1|i}=    \mathbf{W}_{\mathbf{D}} \mathbf{P}_{i+1}  \mathbf{W}_{\mathbf{D}}^\intercal$.
In all, $\hat{\mathbf{Y}}_{i+1}= \mathcal{T}_{\mathbf{D}}( \hat{\mathbf{X}}_{i+1}  )  $  and  $ \mathbf{P}y_{i+1}= \mathbf{W}_{\mathbf{D}} \mathbf{P}_{i+1}  \mathbf{W}_{\mathbf{D}}^\intercal  $. By mathematical induction, we can see 
$  \hat{\mathbf{Y}}_n= \mathcal{T}_{\mathbf{D}} (  \hat{\mathbf{X}}_n ) $ for $n\geq i$ and hence 
the output of  this filter  
 is invariant under any deterministic  transformation $ \mathcal{T}_{\mathbf{D}} $.

\subsection{Proof of Theorem \ref{theorem:invariancesto}}
\label{proof::stochasticVINS}

Here we only prove the sufficient condition.  It is assumed that this filter satisfies:
$ \mathbf{H}_{n+i+1}\pmb{\Phi}_{n+i}\pmb{\Phi}_{n+i-1}\cdots\pmb{\Phi}_{i} \mathbf{N}_i = \mathbf{0} \text{ } \forall\text{ } n\text{ and } i\geq 0$. 

For any  estimate $(\hat{\mathbf{X}}_i,\mathbf{P}_i)$ at time-step $i$, we have another estimate $(\hat{\mathbf{Y}}_{i},\mathbf{P}y_{i})=(   \hat{\mathbf{X}}_i ,  \mathbf{P}_i +  \mathbf{N}_i \pmb{\Sigma}  \mathbf{N}_i^\intercal   )$  after applying the stochastic identify transformation $\mathcal{T}_{\mathbf{S}}$ where $\mathbf{S}= ( \mathbf{0}, \pmb{\epsilon} )$ and 
$\pmb{\epsilon} \sim \mathcal{N}(\mathbf{0}, \pmb{\Sigma})$. After one step propagation, 
we have  $(\hat{\mathbf{X}}_{i+1|i},\mathbf{P}_{i+1|i})$ and $(\hat{\mathbf{Y}}_{i+1|i},\mathbf{P}y_{i+1|i})= (\hat{\mathbf{X}}_{i+1|i},\mathbf{P}_{i+1|i}+ \pmb{\Phi}_i  \mathbf{N}_i \pmb{\Sigma}  \mathbf{N}_i^\intercal   \pmb{\Phi}_i^\intercal )   $. Note that  $ \mathbf{H}_{i+1} \pmb{\Phi}_i  \mathbf{N}_i =\mathbf{0}  $,  we can easily get 
 $(\hat{\mathbf{Y}}_{i+1},\mathbf{P}y_{i+1})= (\hat{\mathbf{X}}_{i+1},\mathbf{P}_{i+1}+ \pmb{\Phi}_i  \mathbf{N}_i \pmb{\Sigma}  \mathbf{N}_i^\intercal   \pmb{\Phi}_i^\intercal     ) $. By mathematical induction, 
 we have $(\hat{\mathbf{Y}}_{n},\mathbf{P}y_{n})= (\hat{\mathbf{X}}_{n},\mathbf{P}_{n}+ \pmb{\Phi}_{n} \cdots \pmb{\Phi}_i  \mathbf{N}_i \pmb{\Sigma}  \mathbf{N}_i^\intercal   \pmb{\Phi}_{i}^\intercal \cdots \pmb{\Phi}_n^\intercal    ) $. Therefore,
 the output of this filter is invariant under any stochastic identify transformation.


\bibliographystyle{IEEEtran}
\bibliography{ref.bib}


\end{document}